\documentclass[10pt,twocolumn,letterpaper]{article}

\usepackage{iccv}
\usepackage{times}
\usepackage{epsfig}
\usepackage{graphicx,epstopdf,subfig}
\usepackage{amsmath,amsthm,amssymb}
\usepackage{amssymb}
\usepackage{color}
\usepackage{url}
\usepackage{comment}
\usepackage{algorithm,algorithmic}
\usepackage[breaklinks=true,bookmarks=false]{hyperref}

\iccvfinalcopy % *** Uncomment this line for the final submission

\def\iccvPaperID{1595} % *** Enter the ICCV Paper ID here

\newtheorem{theorem}{Theorem}
\newtheorem{lemma}{Lemma}
\newtheorem{remark}{Remark}
\newtheorem{definition}{Definition}
\def\ben{\begin{equation}}
\def\een{\end{equation}}
\def\bmx{\begin{bmatrix}}
\def\emx{\end{bmatrix}}
\DeclareMathOperator*{\argmin}{arg\,min}
\def\balpha{{\boldsymbol{\alpha}}}
\def\A{{\mathbf A}}
\def\D{{\mathcal D}}
\def\E{{\mathbb E}}
\def\F{{\mathcal F}}

\def\N{{\mathcal N}}
\def\R{{\mathbb R}}
\def\S{{\mathcal S}}
\def\T{{\mathcal T}}
\def\U{{\mathbf U}}
\def\V{{\mathcal V}}
\def\X{{\mathcal X}}
\def\Y{{\mathcal Y}}
\def\Z{{\mathcal Z}}
\def\s{{\mathbf s}}
\def\u{{\mathbf u}}
\def\v{{\mathbf v}}
\def\x{{\mathbf x}}
\def\y{{\mathbf y}}
\def\z{{\mathbf z}}

\begin{document}

%%%%%%%%% TITLE
\title{Geometry-aware Deep Transform}

\author{Jiaji Huang\qquad\qquad Qiang Qiu~\qquad\qquad Robert Calderbank~\qquad\qquad Guillermo Sapiro\\
Duke University\\
Durham, NC, 27708\\
{\tt\small \{jiaji.huang, qiang.qiu, guillermo.sapiro, robert.calderbank\}@duke.edu }
}

\maketitle
%\thispagestyle{empty}

%%%%%%%%% ABSTRACT
\begin{abstract}
   Many recent efforts have been devoted to designing sophisticated deep learning structures, obtaining revolutionary results on benchmark datasets.   The success of these deep learning methods mostly relies on an enormous volume of labeled  training samples to learn a huge number of parameters in a network; therefore,  understanding  the generalization ability of a learned deep network cannot be overlooked, especially when restricted to a small training set, which is  the case for many applications.
   In this paper, we  propose a novel deep learning objective formulation that unifies both the classification and metric learning criteria. We then introduce a geometry-aware deep transform to enable a non-linear discriminative and robust feature transform, which shows competitive performance on small training sets for  both synthetic and real-world data. We further support the proposed framework with a formal $(K,\epsilon)$-robustness analysis.
\end{abstract}

%%%%%%%%% BODY TEXT
\section{Introduction}

Many recent efforts have been devoted to learning a mapping from low-level image features, e.g.,  image patches~\cite{deepID,deepface}, LBP descriptors~\cite{Cao2013,CSML2010}, to high-level discriminative representations. The learned feature mapping often increases the inter-class separation while reducing the intra-class variation.
This idea dates back at least  to the linear discriminant analysis (LDA) for linear cases;  however,  if we allow the feature mapping to be non-linear, e.g., deep convolutional neural network~\cite{DML,deepID,deepID2}, the discriminability of the learned representation is often significantly enhanced compared to its linear counterpart.

Deep learning techniques achieve unprecedentedly high precision in object and scene classification, where an enormous volume of labeled training samples are often required to learn a rich set of parameters~\cite{largeDNN2012,ImageNet_NIPS,deepID2}.
Despite such revolutionary advances, many real-world classification problems remain challenging, due to the large number of non-linearly separable classes and the scarcity of training samples.
One such example is face verification~\cite{LFWref}, where recently reported successes mostly rely on  huge proprietary training sets, e.g., 4.4 million labeled faces from 4,030 people in \cite{deepface}; however, publicly available training datasets often consist of only a small set of subjects with several samples per subject.
It is a notoriously difficult task to learn from limited training samples a deep structure that can generalize well on testing data \cite{Livni2014}.

While great current attentions are paid to smart manipulation of different deep architectures  for more discriminative representations~\cite{face++,deepID,deepface},  in this paper,  we focus on the generalization problem, i.e., how to encourage a mapping learned from limited training samples to generalize well over testing data. This issue is of significant importance when the training samples are scarce, in which case the network optimized on the training set is likely at the risk of overfitting.
We provide both analytic and experimental illustrations on the generalization errors of a learned deep structure, under several popular objective functions.

We further propose a geometry-aware feature transformation framework, which balances between discriminability and generalization. The proposed framework encourages inter-class separation while at the same time penalizes the distortion of intra-class structure.
This also extends the ``shallow" setup in~\cite{geo-ML} to a deep architecture, also providing theoretical insights regarding robustness. In particular, we show that constraining feature mapping functions to be near-isometry in local sub-regions yields robust algorithms.
We first motivate our framework with a synthetic  example, and then support it through theoretical analysis. We further validate our framework using face verification experiments and report state-of-the-art results on the challenging  LFW face dataset .

Our main contributions are:
\begin{itemize}
\item proposing a novel deep learning objective that unifies the classification and metric learning criteria.
\item providing a theoretical argument showing that awareness of geometry leads to robustness;
\item motivating a general algorithmic framework which considers data geometry in the formulation;
\item designing a learned deep transform, as a particular example of the proposed geometric framework, that achieves state-of-art results.
\end{itemize}

%=======================================================================
\section{Geometry-aware deep transform}	
\label{sec:gdt}
%=======================================================================
Deep networks are often optimized for a classification objective, where class-labeled samples are input as training \cite{largeDNN2012, ImageNet_NIPS, deepID,deepface};
or a metric learning objective, where training data are input as positive and negative pairs \cite{DML, deepID2}. \footnote{A positive pair contains two samples from the same class, and a negative pair contain two  samples from different classes.}
In this section, we first propose a novel deep learning objective that unifies the classification and metric learning criteria.
We then introduce a geometry-aware deep transform, and optimize it through standard  back-propagation.
We further support the proposed framework with a formal $(K,\epsilon)$-robustness analysis \cite{Xu2012}.

\subsection{Pedagogic formulation}

We use the following two-class problem as an illustration example:
The first class is generated as
$\x=\U\v/\|\U\v\|$, where $\v$ is with probability (w.p.) $1/2$ from a constrained plane $-y+z=1, x\in[-1,1], z\in[-3,0]$, and w.p. $1/2$ from plane $y+z=1, x\in[-1,1], z\in[0,3]$.
$\U$ is a $d\times 3$ ($d\gg 3$, $d=100$ in this case) matrix that embeds $\x$ into a $d$-dimensional space.
Similarly, the second class is generated as
$\x=\U\u/\|\U\u\|$, where $\u$ is w.p. $1/2$ from $-y+z=-1, x\in[-1,1], z\in[-3,0]$, and w.p. $1/2$ from $y+z=-1, x\in[-1,1], z\in[0,3]$.
For each class, 40 training and 1000 testing samples are generated.
Fig.~\ref{fig:toydata} visualizes the training and testing data by randomly projecting it to a 3 dimensional coordinate system,  with different colors representing different classes.
Observe that the two classes are not linearly separable, which necessitates a non-linear feature transform.
\begin{figure}[h!]
	\subfloat[Training samples: 40 per class.]{\includegraphics[width=0.5\columnwidth]{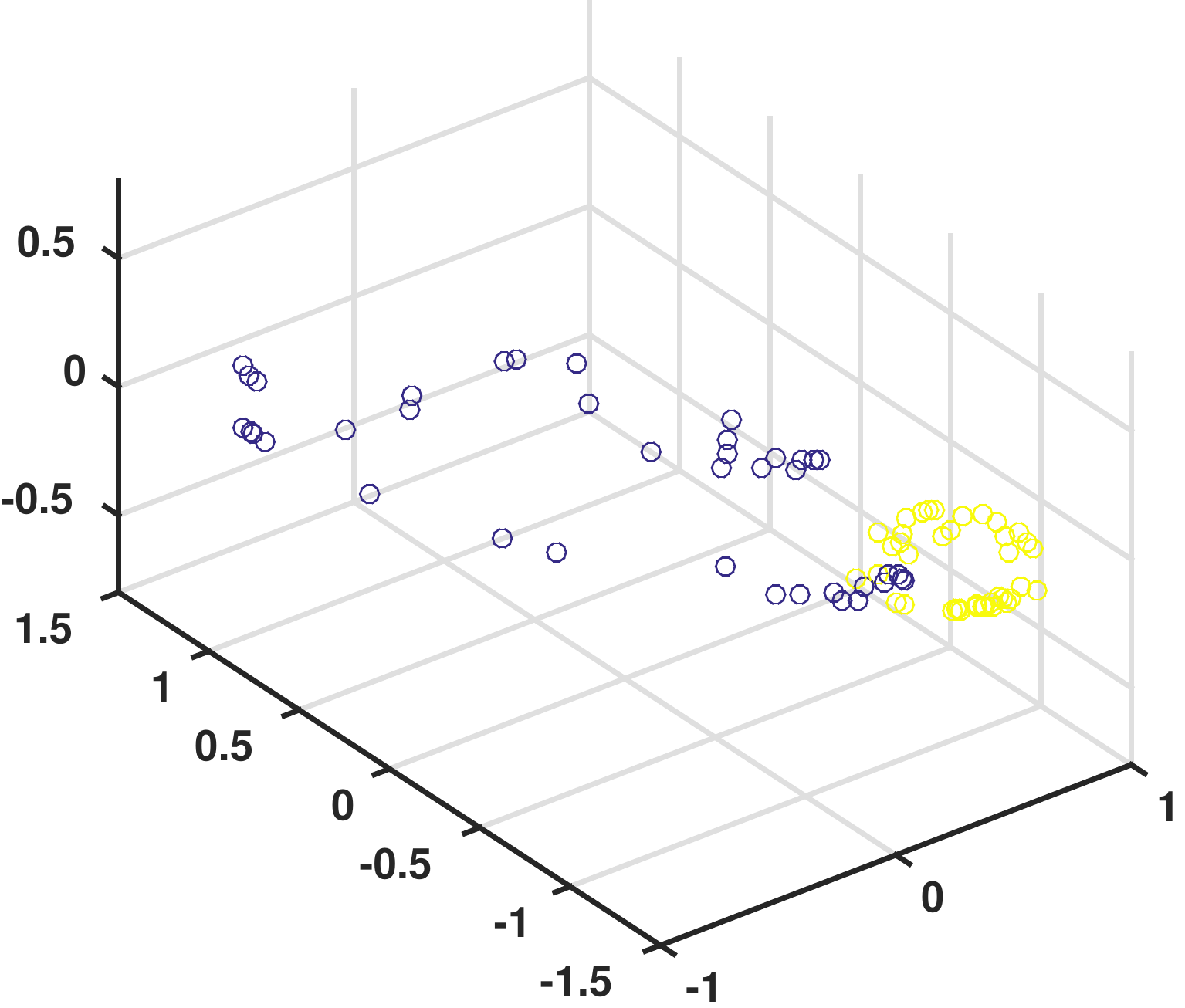}}
	\subfloat[Testing samples: 1000 per class.]{\includegraphics[width=0.5\columnwidth]{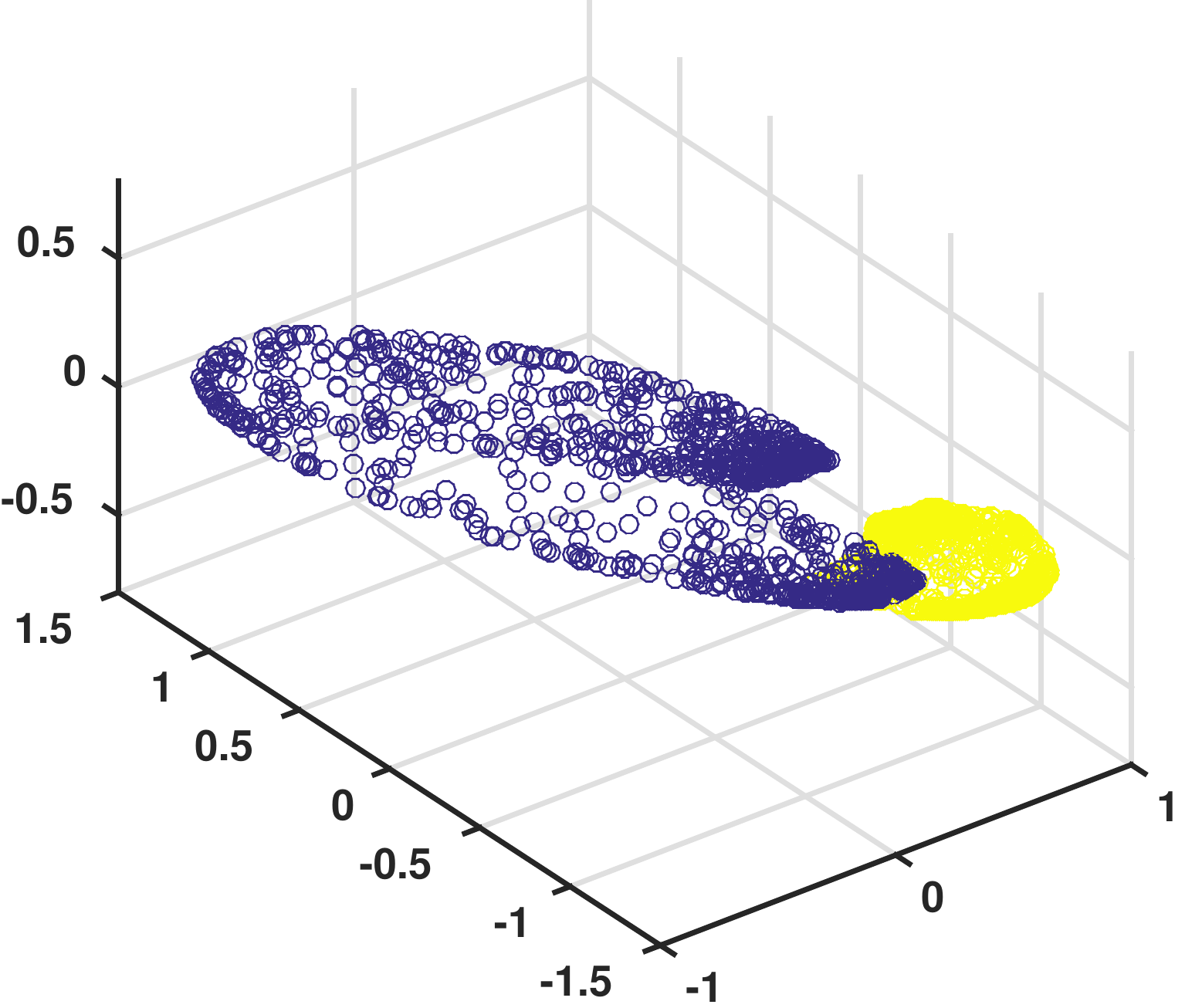}}
	\caption{Training and testing samples.}
	\label{fig:toydata}
\end{figure}
\begin{figure}[h!]
	\subfloat[Transformed training samples.]{\label{fig:naiveForm1_train}\includegraphics[width=0.5\columnwidth]{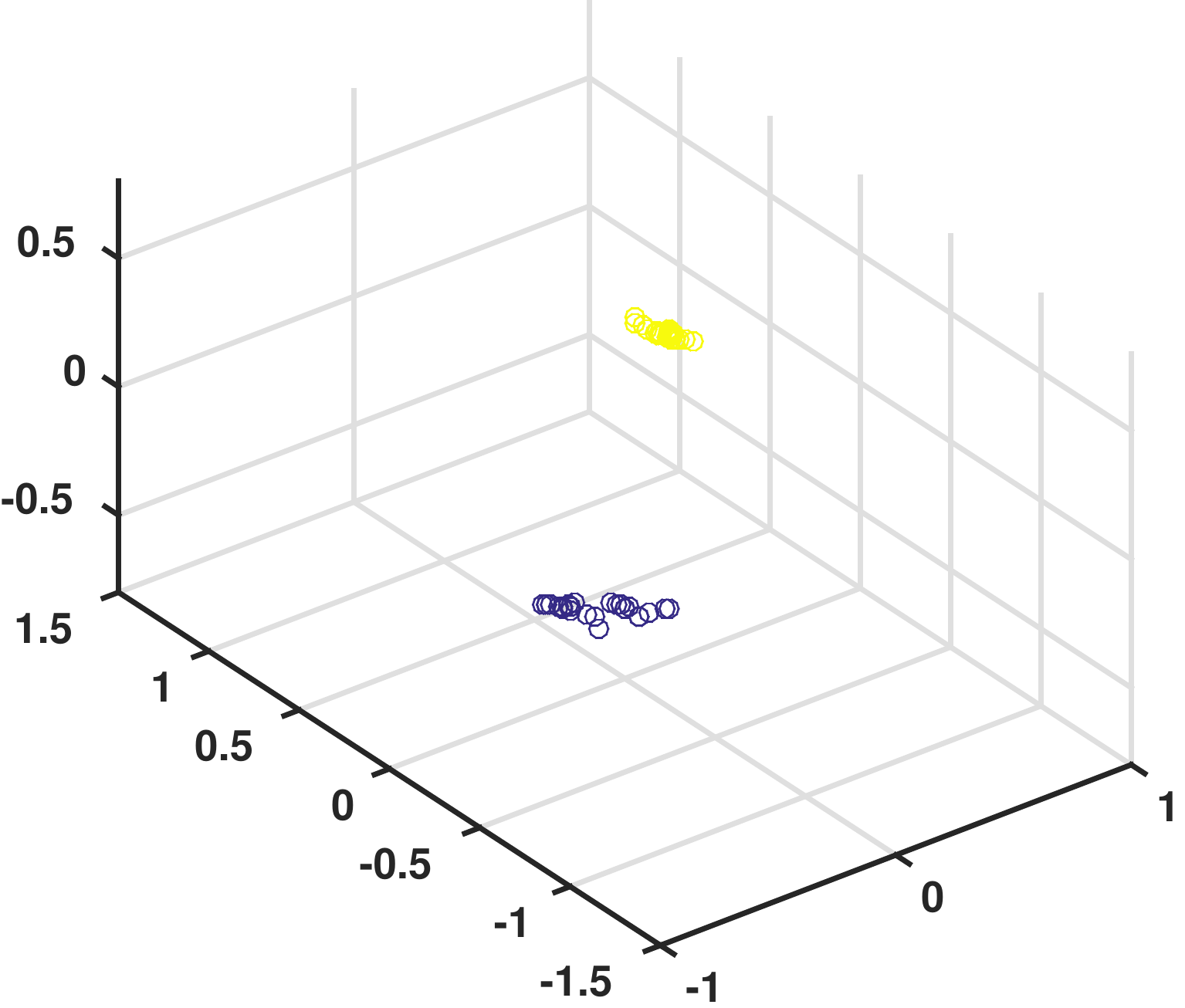}}
	\subfloat[Transformed testing samples.]{\label{fig:naiveForm1_test}\includegraphics[width=0.5\columnwidth]{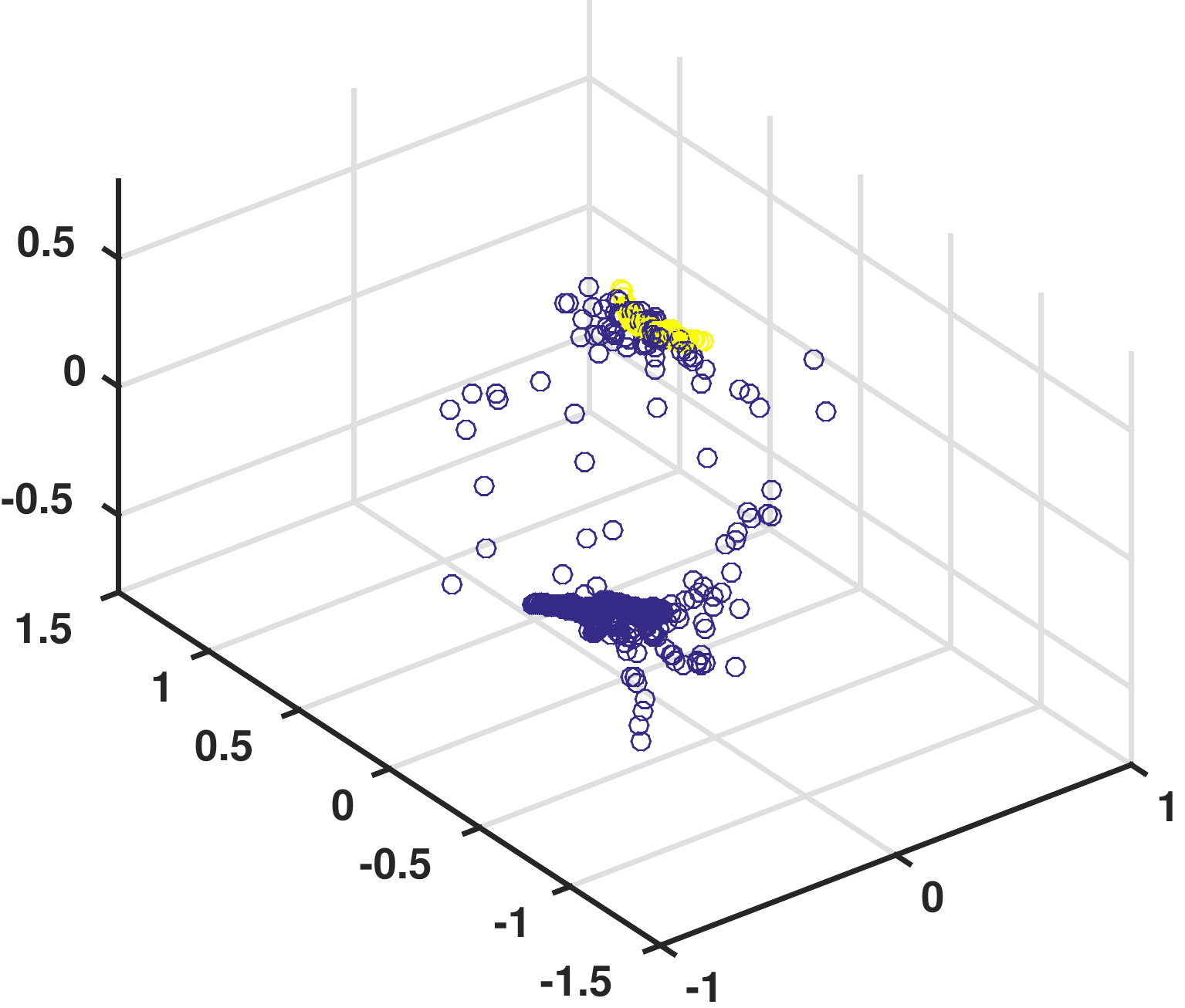}}
	\caption{Transformed features using a metric learning formulation.}
	\label{fig:naiveForm1}
\end{figure}
\begin{figure}[h!]
	\subfloat[Transformed training samples.]{\includegraphics[width=0.5\columnwidth]{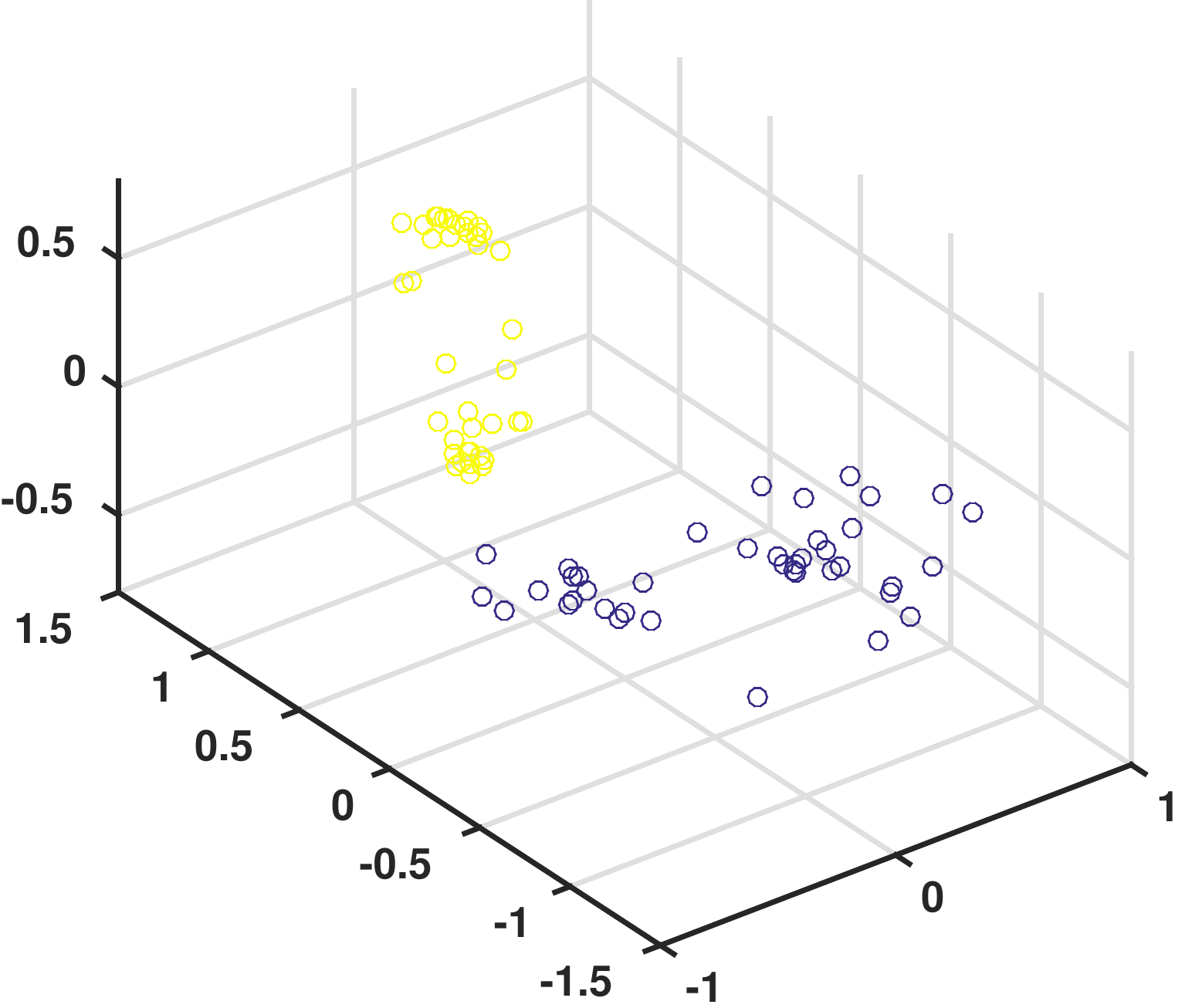}}
	\subfloat[Transformed testing samples.]{\includegraphics[width=0.5\columnwidth]{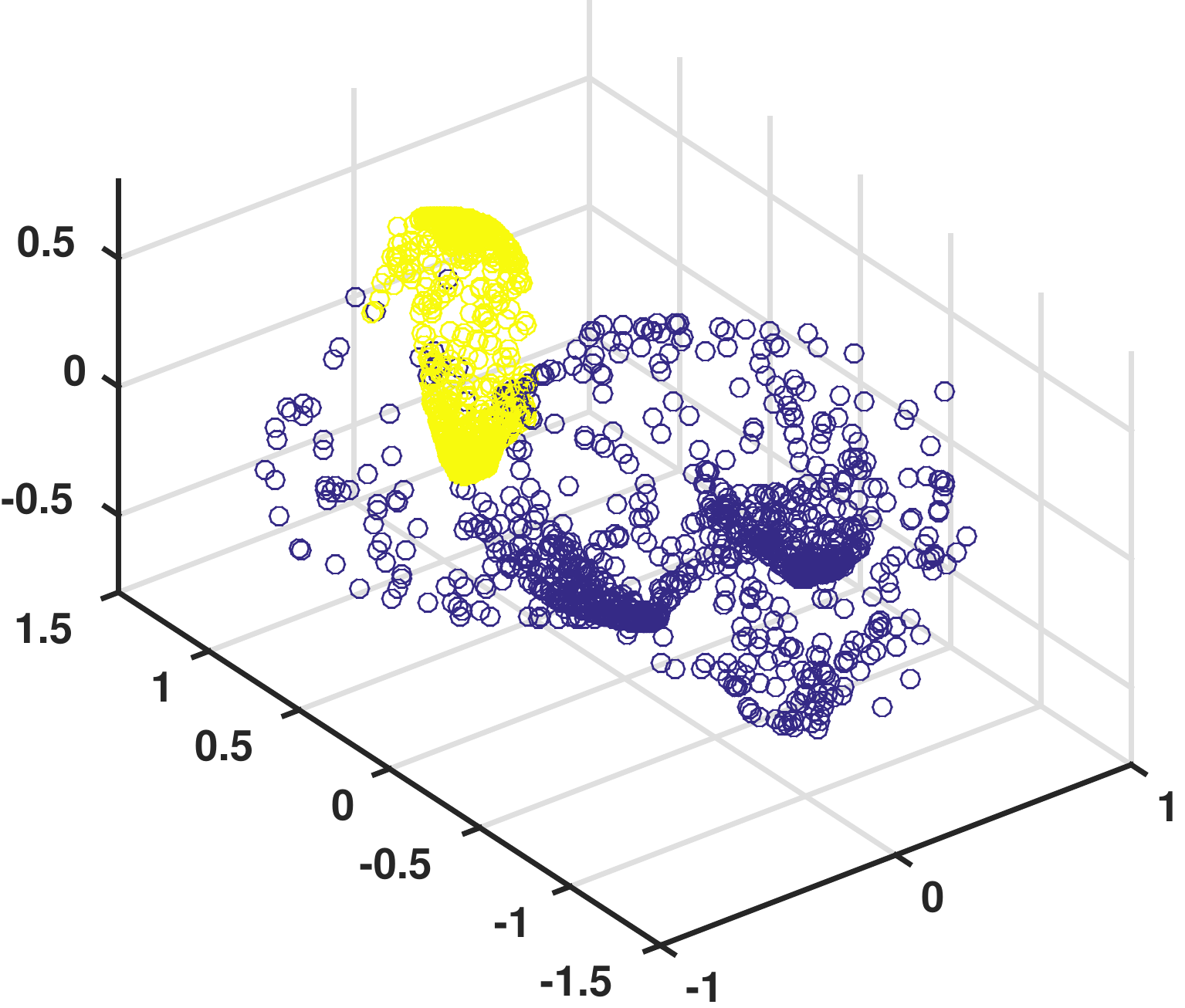}}
	\caption{Transformed features using a classification formulation.}
	\label{fig:naiveForm2}
\end{figure}
\begin{figure}[h!]
	\subfloat[transformed training samples]{\includegraphics[width=0.5\columnwidth]{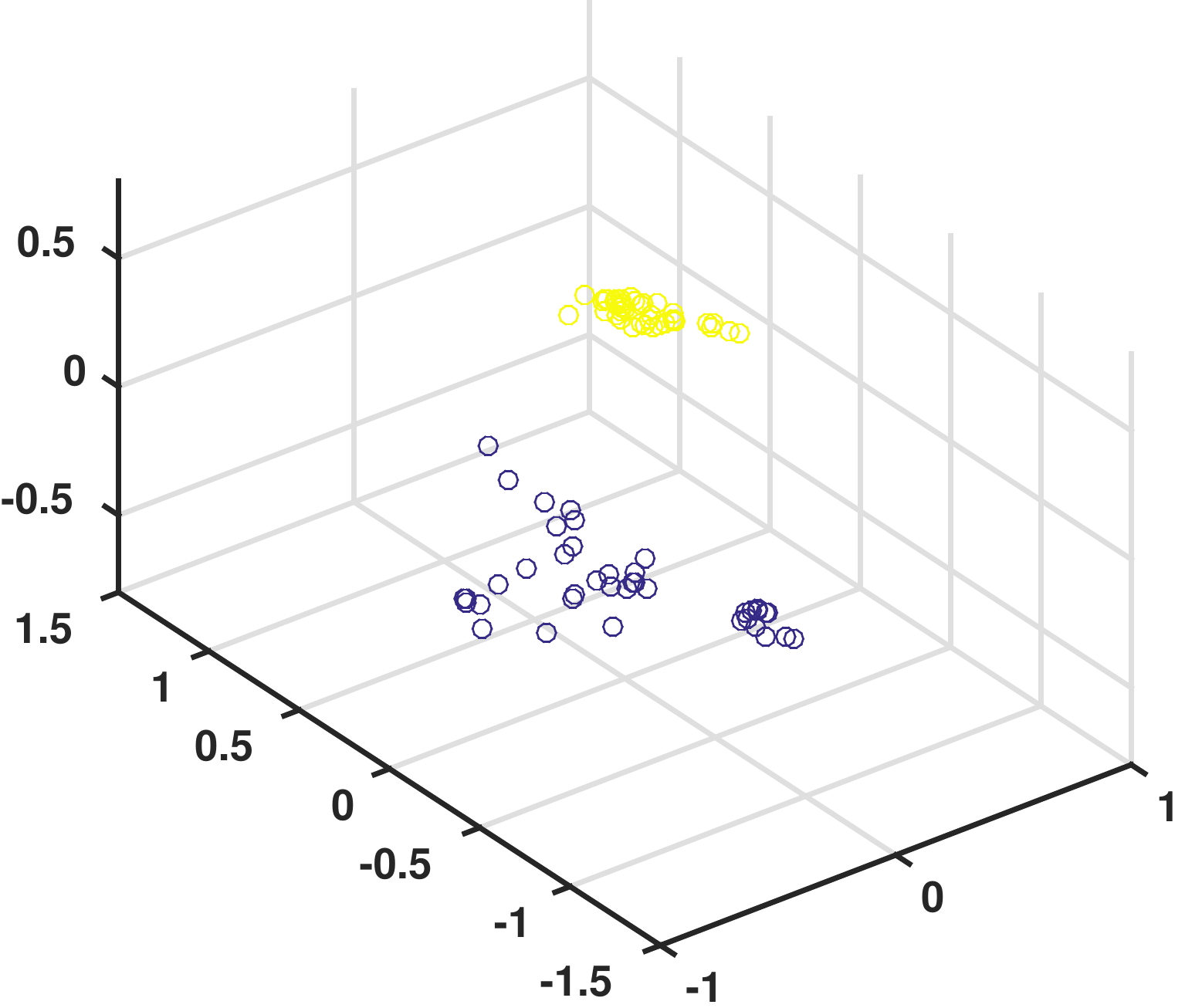}}
	\subfloat[transformed testing samples]{\includegraphics[width=0.5\columnwidth]{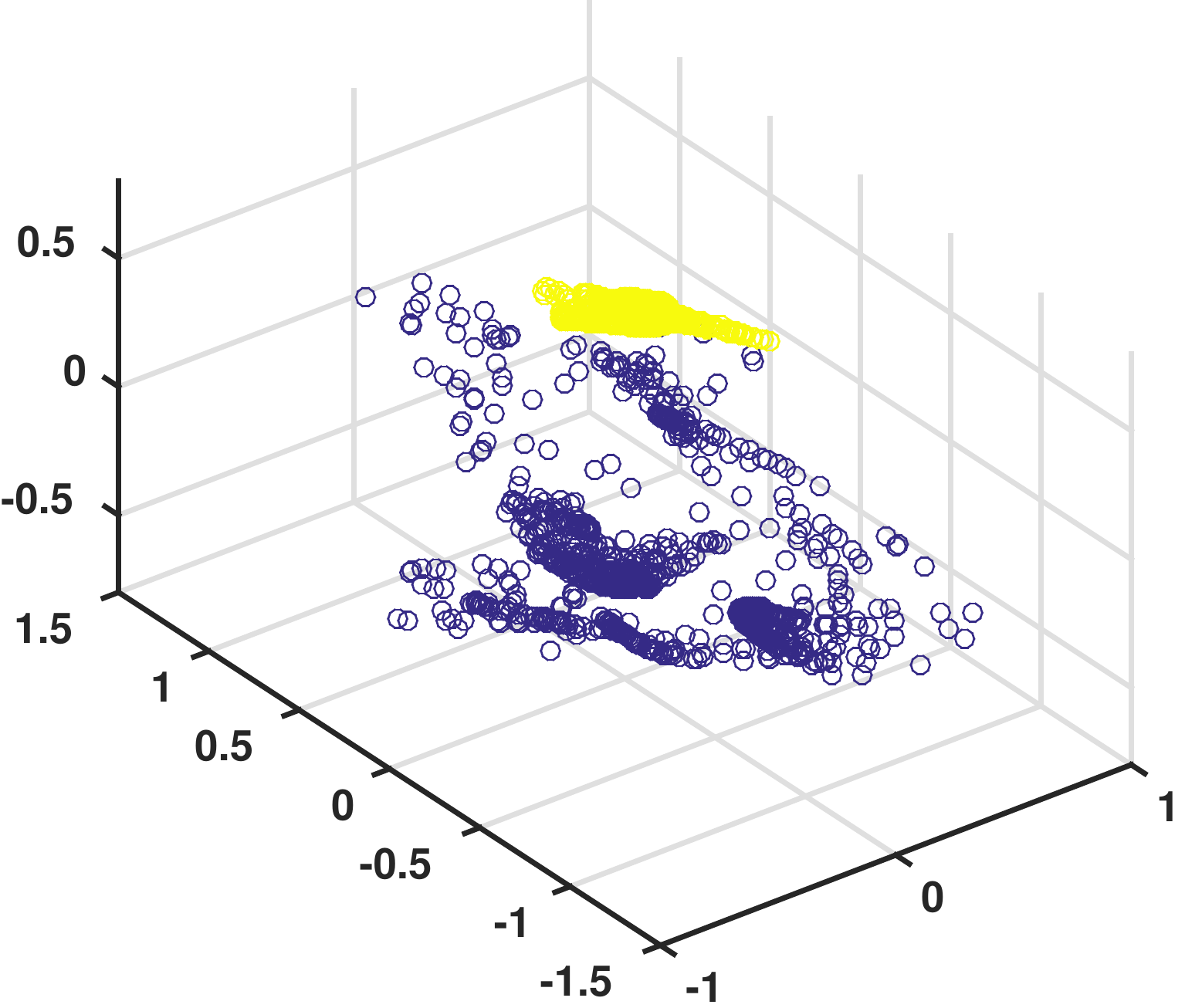}}
	\caption{Transformed features using GDT with $\lambda=0.4$.}
	\label{fig:GDTlambda0.4}
\end{figure}

We want to learn a mapping $f(\x)$ that transforms the low-level feature $\x$ to a more discriminative one.
In this paper, we are particularly interested in non-linear transforms $f(\cdot)$ implemented as a deep neural network.
However, the method and theory we develop are general in the sense that any other family of $f(\cdot)$ can be adopted as well.
In this example, $f(\cdot)$ is implemented as a 2-layer fully connected neural network with $\tanh$ as the squash function,  $f(\cdot)$ taking the form
\ben
f(\x)=\tanh(\A_2\tanh(\A_1\x)),
\een
where $\A_1,\A_2\in\R^{d\times d}$ are the linear coefficients in those two layers.

\noindent \textbf{Metric learning formulation.}  The general goal of metric learning is to ensure that, after the transform, the distance between intra-class points is small, while the inter-class distance is large.
The Euclidean distance is a common choice of metric; however, empirical results~\cite{Cao2013,CSML2010} have shown that the cosine distance outperforms Euclidean distance on certain tasks such as face verification. Moreover, cosine distance is bounded and easier for us to design the loss function.
We therefore adopt the cosine distance in this paper, and propose  the  metric learning formulation
\ben
	\min_{\A_1,\A_2}\sum_{i\neq j} \left(\frac{f(\x_i)^\top f(\x_j)}{\|f(\x_i)\|\cdot\|f(\x_j)\|}-t_{i,j}\right)^2,
	\label{eq:naiveForm1}
\een
where the indicator
\[
t_{i,j}=\left\{ \begin{array}{ll} 1 & \mbox{if } \x_i, \x_j \in\mbox{ same class,}\\ -1 &\mbox{otherwise.} \end{array}\right.
\]
Notice that $\frac{f(\x_i)^\top f(\x_j)}{\|f(\x_i)\|\cdot\|f(\x_j)\|}\in[-1,1]$ is the cosine of the angle between the transformed features $f(\x_i)$ and $f(\x_j)$. The objective of~\eqref{eq:naiveForm1} is to encourage the intra-class angle to be  close to 0, and the inter-class ones to be as separated as  $\pi$.
	We use back-propagation to optimize the parameters, $\A_1$ and $\A_2$, as explained later.

Fig.~\ref{fig:naiveForm1_train} visualizes the transformed training samples by the learned $f(\cdot)$. The learned transform significantly pulls apart the two classes and reduces the variations within each individual class.  We then apply the learned $f(\cdot)$ to the testing samples (fig.~\ref{fig:naiveForm1_test}). However, we observe that the two classes are not well separated, raising our concerns about the robustness of the pure metric learning formulation in ~\eqref{eq:naiveForm1}.

\noindent \textbf{Classification formulation.} Now let us consider a different objective function, where we encourage the intra-class angles to be preserved after the transformation. This new objective has a unified formulation as \eqref{eq:naiveForm1}, but now  the indicator becomes
\[
t_{i,j}=\left\{ \begin{array}{ll} \frac{\x_i^\top \x_j}{\|\x_i\|\|\x_j\|} & \mbox{if } \x_i, \x_j \in\mbox{ same class,}\\ -1 &\mbox{otherwise.} \end{array}\right.
\]

We denote this objective function as a classification formulation, as it shares similar attributes to the classification objective commonly optimized for a deep network \cite{deepID,deepface}. Explicit constraints are imposed to separate different classes, e.g., $t_{i,j}=-1$ for negative pairs here, but only weak constraints are used to assign similar representation to the same class.
This classification formulation is less ambitious than the metric learning formulation, as it does not require the variance in the same class being reduced. $f(\cdot)$ is implemented as the 2-layer neural network as described before, and  optimized through back-propagation.

The transformed training and testing samples are visualized in Fig.~\ref{fig:naiveForm2}.
Comparing Fig.~\ref{fig:naiveForm1} and Fig.~\ref{fig:naiveForm2}, we observe that although our metric learning formulation works well on the training data, it does not well discriminate the two classes on testing data, \textit{i.e.}, it has a big generalization error.
In contrast, following the classification formulation, the intra-class variance is not reduced, yet the deterioration from training to testing is not so significant.
In other words, while the  metric learning formulation is too optimistic about the discrimination we can achieve, the classification formulation is more robust but conservative.

%--------------------------------------------------------------------
\subsection{Proposed formulation and algorithm}
%--------------------------------------------------------------------

  We introduce now a geometry-aware deep transform.
   We  use $f_\balpha(\cdot)$ to denote the feature transform,
   to emphasize that  $\balpha$ are parameters to be learned, e.g.,
filters in a neural network ($\A_1, \A_2$ in the previous section).
   $f_\balpha$ can be a linear function or a non-linear function implemented by a neural network.

We formulate the transformation learning problem as:
\ben
	\min_{\balpha}{1\over2}\sum_{i\neq j} \left(\frac{f_\balpha(\x_i)^\top f_\balpha(\x_j)}{\|f_\balpha(\x_i)\|\cdot\|f_\balpha(\x_j)\|}-t_{i,j}\right)^2,
	\label{eq:obj}
\een
where the indicator
\[
t_{i,j}=\left\{ \begin{array}{ll}\lambda+(1-\lambda)\frac{\x_i^\top \x_j}{\|\x_i\|\|\x_j\|} & \mbox{if } \x_i, \x_j \in \mbox{ same class,}\\ -1 &\mbox{otherwise.} \end{array}\right.
\]
and $\lambda\in[0,1]$.
We denote formulation~\eqref{eq:obj} as Geometry-aware Deep Transform (GDT).
The GDT objective is a weighted combination of the two pedagogic formulations discussed above. We can understand it as regularizing the metric learning formulation using the classification one.

We use gradient descent (back-propagation if $f_\balpha(\cdot)$ is a deep neural network) to solve for the $\balpha$ in \eqref{eq:obj}. In particular, let us denote the objective in \eqref{eq:obj} as $J$ and define
\ben\begin{aligned}
f_{\balpha}(\x_i) &\triangleq& \y_i,\\
\frac{f_{\balpha}^\top(\x_i)f_{\balpha}(\x_i)}{\|f_\balpha(\x_i)\|\cdot\|f_\balpha(\x_j)\|} &\triangleq& C_{i,j}.
%\lambda+(1-\lambda)\frac{\x_i^\top \x_j}{\|\x_i\|\|\x_j\|} &\triangleq& t_{i,j}.
\end{aligned}\een
Then we have
\ben
\begin{aligned}
\frac{\partial J}{\partial \y_i}=& {1\over\|\y_i\|}
\sum_{j\neq i} \left(C_{i,j}-t_{i,j}\right)
\left[{\y_j\over\|\y_j\|}-C_{i,j}\cdot{\y_i\over\|\y_i\|}  \right].
\end{aligned}
\label{eq:dJdy}
\een
${\partial J\over\partial \y_j}$ can be calculated in the same manner.
Then we back-propagate this gradient through the network to update all the parameters.
More specifically, we denote $\balpha^{(k)}$ as the filter weights and bias in the $k$-th ($1\leq k\leq K$) layer.
And $\x_i^{(k)}$ as the output of the $k$-th layer excited by the input $\x_i^{(k-1)}$
(therefore $\y_i=\x_i^{(K)}$ and $\x_i=\x_i^{(0)}$).
Then,
\ben
\begin{aligned}
	{\partial J\over\partial\balpha^{(K)}} &= \sum_i{\partial J\over\partial \y_i}\cdot {\partial \y_i \over \partial\balpha^{(K)}},\\
	{\partial J\over\partial\balpha^{(k)}} &= \sum_i{\partial J\over\partial \x_i^{(k+1)}}\cdot{\partial \x_i^{(k+1)}\over\partial \x_i^{(k)}}\cdot {\partial \x_i^{(k)} \over \partial\balpha^{(k)}}, 1\leq k\leq K-1.
\end{aligned}
\label{eq:dJdalpha}
\een
An overview of the GDT algorithm is summarized in Algorithm~\ref{algbox:GDT}.
\begin{algorithm}[h!]
	\caption{Gradient descent solver for GDT}
	\begin{algorithmic}
		\REQUIRE $\lambda\in[0,1]$, training pairs $\{(\x_i,\x_j,\ell_{i,j})\}$,\\
				a defined $K$-layer network ($f_\balpha(\cdot)$ family), stepsize $\gamma$
		\ENSURE $\balpha$
		\WHILE{stable objective not achieved}
		\STATE Compute $\y_i=f_\balpha(\x_i)$ by a forward pass
		\STATE Compute objective $J$
		\STATE Compute $\partial J\over \partial \y_i$ as Eq.~\eqref{eq:dJdy}
		\FOR {$k=K$ down to $1$}
			\STATE Compute $\partial J\over\partial\balpha^{(k)}$ as Eq.~\eqref{eq:dJdalpha}
			\STATE $\balpha^{(k)}\gets \balpha^{(k)}-\gamma {\partial J\over\partial\balpha^{(k)}}$
		\ENDFOR
		\ENDWHILE
	\end{algorithmic}
	\label{algbox:GDT}
\end{algorithm}

For an illustration of Algorithm~\ref{algbox:GDT}, we apply it with $\lambda=0.4$ to the illustrative example above. The transformed training and testing samples are shown as Fig.~\ref{fig:GDTlambda0.4}. Compared with the two pedagogic formulations (equivalent to $\lambda=1$ and $0$ in GDT respectively), this $\lambda=0.4$ case is balancing between discriminability and robustness.
Before more detailed experimental analysis are shown in Section~\ref{sec:expr},
we now provide theoretical insights to support our robustness claim.

%=======================================================================
\subsection{$(K,\epsilon)$-robustness}
\label{sec:theory}
%=======================================================================

GDT regularizes discriminative transform learning with  intra-class structure preservation. 
%and the terminology for structure preservation is isometry.
In this section, we formally show that a local isometry regularization induces robustness.
In the following, we assume a general objective that works with distance metrics of pairs of transformed features.
%Even though the cosine distance used in GDT is not a metric.
%The theoretical results to be developed in this section provide a way of understanding the observations in section~\ref{sec:motivate}.

Let the low-level feature space be $\X$, and the class label set be $\Y=\{1,\dots,L\}$, where $L$ is the number of classes.
$\Z=\X\times\Y$ is the set of low-level features and their corresponding labels.
The training set is
\[
\T=\{(\x_1,y_1),\dots,(\x_n,y_n)\}\triangleq\{\z_1,\dots,\z_n\}\in\Z^n,
\]
which consists of $n$ i.i.d. samples drawn from an unknown distribution $\D$ defined on $\Z$.
The feature mapping is $f_\balpha(\x): \X\mapsto \F$, where $\F$ is the transformed feature space.

Denote $\rho$ as a metric endowed with $\X$ and $\F$.
Define pair label $\ell_{i,j}=1$ if $y_i=y_j$, and $-1$ otherwise.
We may adopt a loss function $g(\rho(f_\balpha(\x_i),f_\balpha(\x_j)),\ell_{i,j})$ that encourages $\rho(f_\balpha(\x_i),f_\balpha(\x_j))$ to be small (big) if $\ell_{i,j}=1$ ($-1$).
We require the Lipschtiz constant of $g(\cdot,1)$ and $g(\cdot,-1)$ to be upper bounded by $A$ ($0<A<\infty$).
Examples of such $g$ include the hinge loss
\ben
\max(-\ell_{i,j}(1-\rho(f_\balpha(\x_i),f_\balpha(\x_j))),0),
\label{eq:ml1}
\een
as well as its smoothed version
\ben
\log(1+e^{-\ell_{i,j}(1-\rho(f_\balpha(\x_i),f_\balpha(\x_j)))}),
\label{eq:ml2}
\een
both of which are commonly adopted in the metric learning literature~\cite{DML}.
In GDT formulation~\eqref{eq:obj}, the quadratic loss has bounded Lipschtiz w.r.t. the cosine distance $C_{i,j}$ as well.
In the following, we denote
\[
g(\rho(f_\balpha(\x_i),f_\balpha(\x_j)),\ell_{i,j})\triangleq h_\balpha(\z_i,\z_j)
\]
for short.

The empirical loss on the training set (associated with parameter $\balpha$) is
\ben
 R_{emp}(\balpha) \triangleq {2\over n(n-1)} \sum_{\stackrel{i,j=1}{i\neq j}}^n h_\balpha(\z_i,\z_j).
\een
And the expected loss is
\ben\begin{aligned}
	R(\balpha) &\triangleq& \E_{\z_1^\prime,\z_2^\prime\sim\D} \left[h_\balpha(\z_1^\prime,\z_2^\prime)\right].\\
\end{aligned}\een
The algorithm is a program that seeks
\ben
\balpha_\T \triangleq \argmin_\balpha R_{emp}(\balpha),
\label{eq:algo}
\een
which minimizes the empirical loss on the training set $\T$.
A metric learning type formulation, including our GDT, falls in the category of algorithm~\eqref{eq:algo}.
The quantity $R_{emp}(\balpha_\T)-R(\balpha_\T)$ is called the algorithm's generalization error.
Smaller generalization error implies robustness.

The work \cite{Xu2012} proposes a notion called $(K,\epsilon)$-robustness, and 
\cite{Bellet2014} extends the definition of robustness to algorithms like \eqref{eq:algo} that work on pairs of samples. It also shows that $(K,\epsilon)$-robust algorithms have generalization error bounded as
\ben
R_{emp}(\balpha_\T)-R(\balpha_\T)\leq \epsilon+O\left(\sqrt{K\over n}\right).
\label{eq:bound}
\een
We now rephrase the definition of $(K,\epsilon)$-robustness in~\cite{Bellet2014}:
\begin{definition}
	The algorithm~\eqref{eq:algo} is $(K,\epsilon)$-robust if $\Z$ can be partitioned into $K$ disjoint set, $\{C_k\}_{k=1}^K$, such that for all $\T\in\Z^n$, the learned $\balpha_\T$ satisfies:\\
	$\forall \z_i, \z_j \in\T$ where $i\neq j$,\\
	$\forall \z_1^\prime,\z_2^\prime\in\Z$, \\
	If $\z_i,\z_1^\prime\in C_p$, and $\z_j,\z_2^\prime\in C_q$ for any $p,q\in\{1,\dots,K\}$,
	then
	\[
		\left|h_{\balpha_\T}(\z_i,\z_j)-h_{\balpha_\T}(\z_1^\prime,\z_2^\prime)\right|\leq \epsilon.
	\]
\label{def1}
\end{definition}
According to Definition~\ref{def1}, the $(K,\epsilon)$-robustness essentially requires that with the learned $\balpha_\T$,
a testing pair $(\z_1^\prime, \z_2^\prime)$ incurs a similar loss with any training pair $(\z_i, \z_j)$ that is in the same subset (in a pair-wise sense).
And according to the generalization error bound~\eqref{eq:bound},
the smaller $\epsilon$ is, the smaller the generalization error tends to be; therefore the more robust the algorithm is.

Before presenting our theory, we need to introduce the covering number, defined as follows:
\begin{definition}
	For a metric space $(\S,\rho)$, we say that $\hat\S\subset \S$ is a $\gamma$-cover of $\S$,
	if $\forall \s\in\S$, $\exists\hat\s\in\hat\S$ such that $\rho(\s,\hat\s)\leq\gamma$.
	The $\gamma$-covering number of $\S$ is
	\[
		\N_{\gamma}(\S,\rho) = \min\{|\hat\S|:\hat\S \mbox{ is a }\gamma\mbox{-cover of }\S\}
	\]
\end{definition}
\begin{remark}
	The covering number describes how many balls (in $\rho$ metric sense) we need to ``cover" a space. 
	Feature space of certain property, e.g., Gaussian distributed, sparsely representable~\cite{coverNumberGauss}, has certain covering number.
	The more complex the feature space is, the more balls we need to cover it.
	In a word, covering number reflects the geometry of the set $\S$.
	In particular, we notice that the set $\S$ with covering number $\N_{\gamma/2}(\S,\rho)$ can be partitioned into $\N_{\gamma/2}(\S,\rho)$ disjoint subsets, such that any two points within the same subset are separated by no more than $\gamma$.
	\label{remark:coverNum}
\end{remark}

\begin{lemma}
$\Z$ can be partitioned into $L\N_{\gamma/2}(\X,\rho)$ subsets, denoted as $\Z_1,\dots,\Z_{L\N_{\gamma/2}(\X,\rho)}$,
such that for all $\z_1\triangleq(\x_1,y_1), \z_2\triangleq(\x_2,y_2)$ belonging to any one of these subsets, $y_1=y_2$ and $\rho(\x_1,\x_2)\leq \gamma$.
\label{lemma1}
\end{lemma}
\begin{proof}
	As noticed immediately after Definition~\ref{def1}, we can partition $\X$ into $\N_{\gamma/2}(\X,\rho)$ disjoint subsets, each with diameter no bigger than $\gamma$.
	Then we can partition $\Z=\X\times\Y$ into $L\N_{\gamma/2}(\X,\rho)$ disjoint subsets, such that any two samples $(\x_1,y_1), (\x_2,y_2)$ in any one of these subsets have $y_1=y_2$ and $\rho(\x_1,\x_2)\leq\gamma$.
\end{proof}
Lemma~\ref{lemma1} also implies a partition of $\X$, denoted as $\X_1,\dots,\X_{L\N_{\gamma/2}(\X,\rho)}$ such that any $\x_i,\x_j$ from the same subset have $\rho(\x_i,\x_j)\leq\gamma$ and share the same label.
\begin{theorem}
	If $f_\balpha(\x)$ is a $\delta$-isometry (\textit{i.e.}, distance distorted by at most $\delta$ after the transform) within each of $\X_1,\dots,\X_{L\N_{\gamma/2}(\X,\rho)}$ as described above, then an algorithm in the category of \eqref{eq:algo} is $(L\N_{\gamma/2}(\X,\rho),2A(\gamma+\delta))$-robust.
	\label{theorem1}
\end{theorem}
\begin{proof}
	 The proof follows the definition of $(K,\epsilon)$-robustness.
	 We pick any training samples $\z_i,\z_j$ and testing samples $\z_1^\prime,\z_2^\prime$ such that $\z_i,\z_1^\prime\in\Z_p$ and $\z_j,\z_2^\prime\in\Z_q$ for some $p,q\in\{1,\dots,L\N_{\gamma/2}(\X,\rho)\}$.
	 Then
	 \[
	 \rho(\x_i,\x_1^\prime)\leq \gamma \mbox{ and } \rho(\x_j,\x_2^\prime)\leq \gamma.
	 \]
	 Notice that $\x_i,\x_1^\prime\in\X_p$ and $\x_j,\x_2^\prime\in\X_q$.
	 Therefore by the $\delta$-isometry definition,
	 \[
		 |\rho(f_\balpha(\x_i),f_\balpha(\x_1^\prime))-\rho(\x_i,\x_1^\prime)|\leq \delta,
	 \]
	 and
	 \[
	 |\rho(f_\balpha(\x_j),f_\balpha(\x_2^\prime))-\rho(\x_j,\x_2^\prime)|\leq \delta.
	 \]
	  Rearranging the above gives
	 \[
		 \rho(f_\balpha(\x_i),f_\balpha(\x_1^\prime))\leq \rho(\x_i,\x_1^\prime)+\delta\leq \gamma+\delta,
	 \]
	 and
	 \[
		 \rho(f_\balpha(\x_j),f_\balpha(\x_2^\prime))\leq \rho(\x_j,\x_2^\prime)+\delta\leq \gamma+\delta.
	 \]
	 We need to bound the difference between $\rho(f_\balpha(\x_i),f_\balpha(\x_j))$ and $\rho(f_\balpha(\x_1^\prime),f_\balpha(\x_2^\prime))$ so that we can further invoke the finite Lipschtiz assumption to bound the quantity $|h_\balpha(\z_i,\z_j)-h_\balpha(\z_1^\prime,\z_2^\prime)|$.
	 Specifically,
	 \[\begin{aligned}
	 &|\rho(f_\balpha(\x_i),f_\balpha(\x_j))-\rho(f_\balpha(\x_1^\prime),f_\balpha(\x_2^\prime))|\\
	 &\leq |\rho(f_\balpha(\x_i),f_\balpha(\x_j))-\rho(f_\balpha(\x_1^\prime),f_\balpha(\x_j))|\\
	 & \quad +|\rho(f_\balpha(\x_1^\prime),f_\balpha(\x_j))-\rho(f_\balpha(\x_1^\prime),f_\balpha(\x_2^\prime))|\\
	 &\leq \rho(f_\balpha(\x_i),f_\balpha(\x_1^\prime))+\rho(f_\balpha(\x_j),f_\balpha(\x_2^\prime))\\
	 &\leq 2(\gamma+\delta),
	 \end{aligned}
	 \]
	 where the second line follow from the triangle inequality, while the third line follows the definition of metric.
	 Notice that $y_i=y_1^\prime$ and $y_j=y_2^\prime$.
	 Therefore $|h_\balpha(\z_i,\z_j)-h_\balpha(\z_1^\prime,\z_2^\prime)|$ is either
	 \[
	 |g(\rho(f_\balpha(\x_i),f_\balpha(\x_j)),1)-g(\rho(f_\balpha(\x_1^\prime),f_\balpha(\x_2^\prime)),1)|,
	 \]
	 or
	 \[
	 |g(\rho(f_\balpha(\x_i),f_\balpha(\x_j)),-1)-g(\rho(f_\balpha(\x_1^\prime),f_\balpha(\x_2^\prime)),-1)|.
	 \]
	 Since the Lipschtiz constants  of $g(\cdot,1)$ and $g(\cdot,-1)$ are no bigger than $A$, we have
	 \[\begin{aligned}
		 & |h_\balpha(\z_i,\z_j)-h_\balpha(\z_1^\prime,\z_2^\prime)|\\
		 & \leq A|\rho(f_\balpha(\x_i),f_\balpha(\x_j))-\rho(f_\balpha(\x_1^\prime),f_\balpha(\x_2^\prime))|\\
		 & \leq2A(\gamma+\delta),
	 \end{aligned}\]
	 which concludes the proof.
\end{proof}
\begin{remark}
	Theorem~\ref{theorem1} tells us that the algorithm will be robust if we constrain the function $f_\balpha(\cdot)$ to be near isometric in local regions.
	And the robustness depends on how much of an isometry $f_\balpha(\cdot)$ is in the local regions.
	The local regions are jointly defined by the class labels and the covering number, which, as we described in remark~\ref{remark:coverNum}, depicts the geometry of the low-level feature space.
	Given that the algorithm is $(K,2(\gamma+\delta))$ robust, by Eq.~\eqref{eq:bound}, we can bound the generalization error of algorithms that belongs to the category of~\eqref{eq:algo} by
	\[
		R_{emp}(\balpha_\T)-R(\balpha_\T)\leq 2(\gamma+\delta)+O\left(\sqrt{K\over n}\right).
	\]
\end{remark}
\begin{remark}
	In practice, we may resort to a formulation like GDT to encourage the mapping $f_\balpha(\cdot)$ to be near isometry in the local regions. 
	We can understand GDT as with small $\delta$, resulting in small generalization error. This explains why GDT is more robust (Fig.~\ref{fig:naiveForm1} to \ref{fig:GDTlambda0.4}) than the metric learning formulation.
\end{remark}	
\begin{remark}
	In fact, GDT only partitions the $\X$ space into $L$ subsets, implicitly assuming a trivial covering number of 1.
	One could further partition within each classes, corresponding to a nontrivial covering number.
	However, this is at the cost of learning local neighborhoods within each class, which is beyond the scope of this paper.
\end{remark}

\section{Experiments}
\label{sec:expr}

We provided a formal analysis in Section~\ref{sec:theory} to support the proposed geometry-aware deep transform as a robust framework for optimizing a deep network. 
In this section, we further present an experimental evaluation of GDT demonstrating its power in producing discriminative and robust features for classification.
We compare GDT with two state-of-the-art deep learning objectives: DeepFace (DF)~\cite{deepface} and Deep Metric Learning (DML)~\cite{DML}.
As discussed before, DeepFace shares attributes with our pedagogic classification formulation, and DML is close to  our  pedagogic metric learning formulation.

%GDT unifies both objectives, and shows more discriminative than DeepFace and more robust than DML.

\begin{figure*}[ht]
\center
	\subfloat[$R_{emp}$ and $\hat R$]{\label{fig:loss}\includegraphics[width=0.32\textwidth]{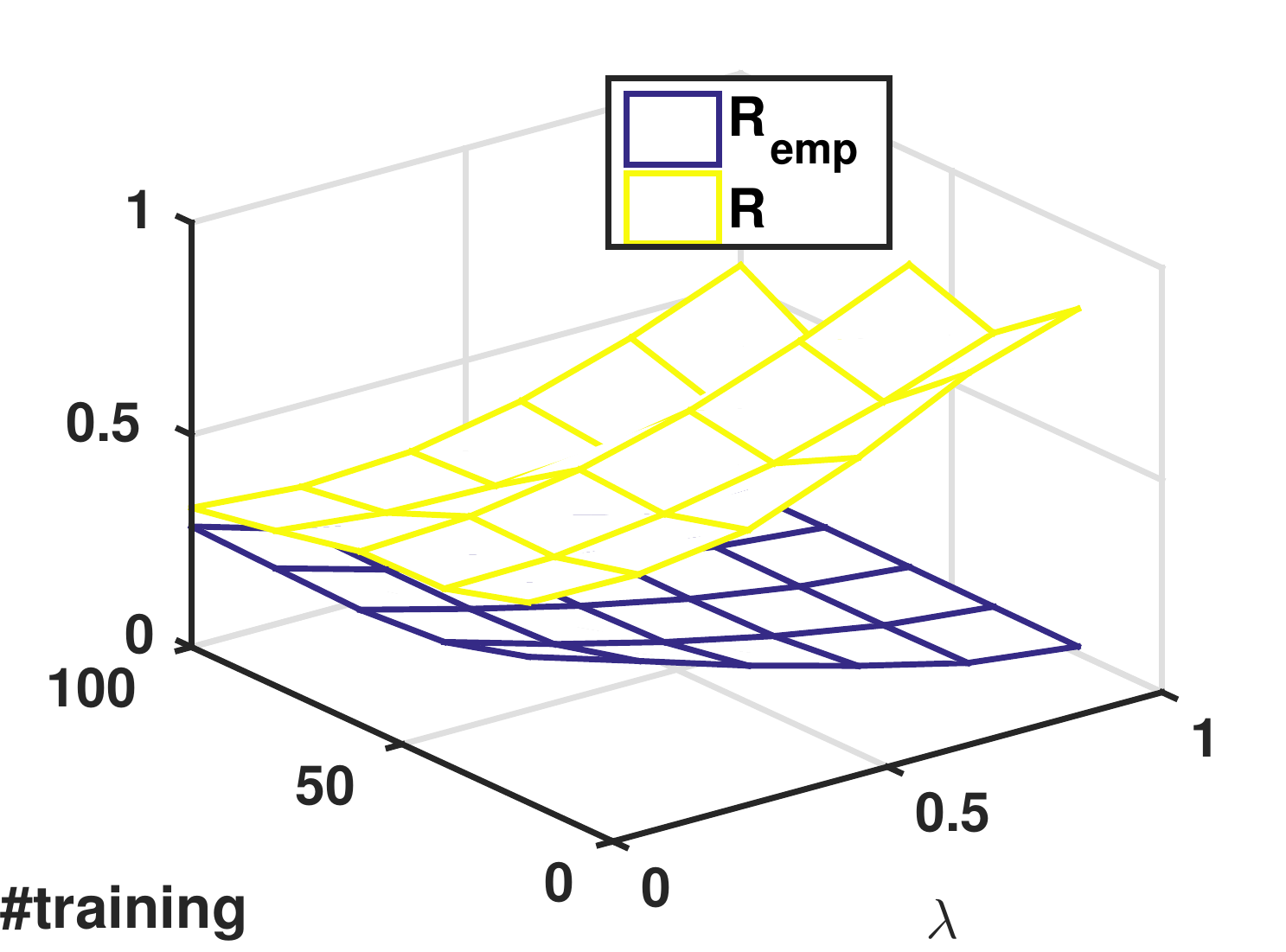}}
	\subfloat[Generalization error]{\label{fig:genErr}\includegraphics[width=0.32\textwidth]{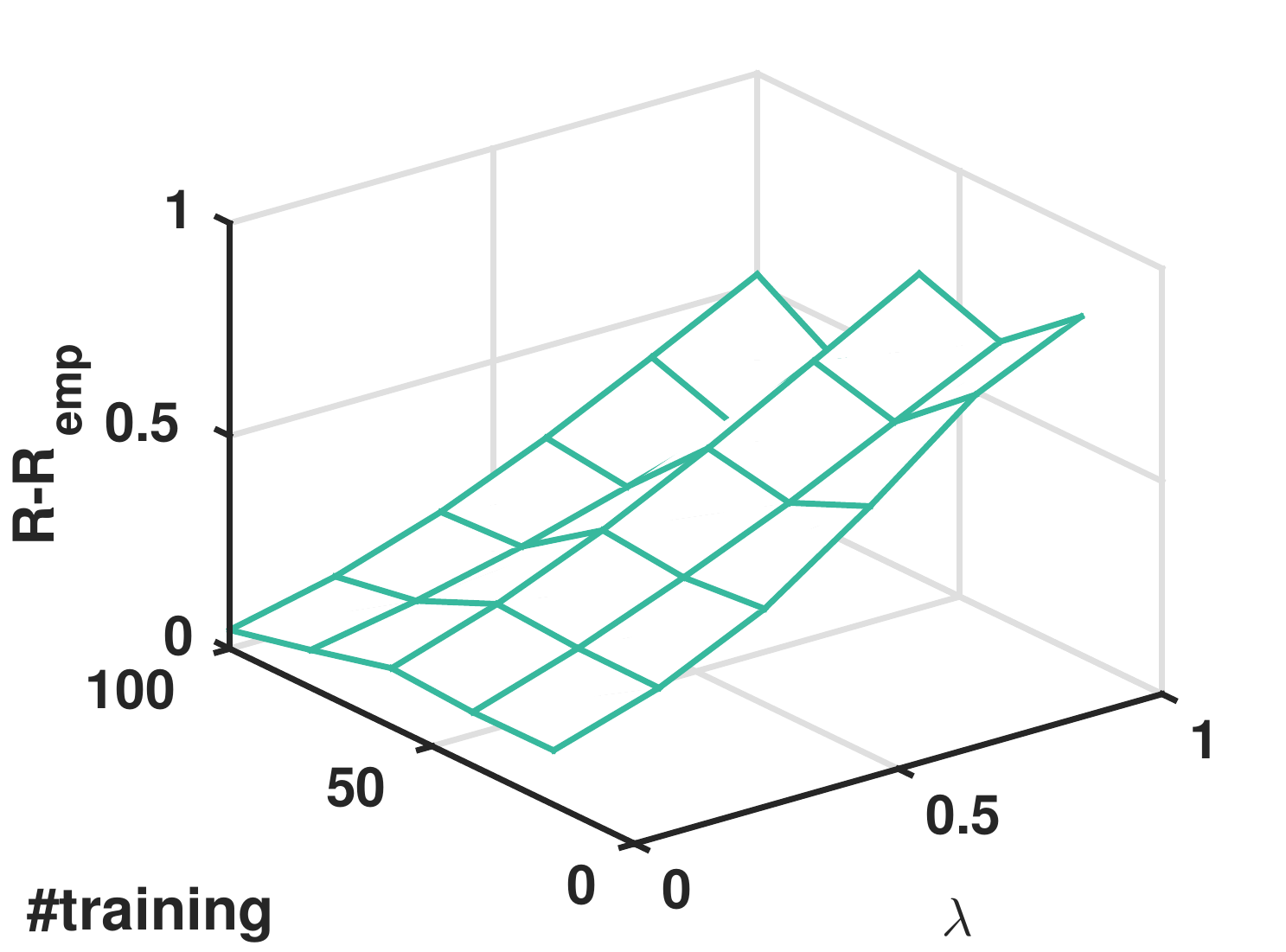}}
    \subfloat[1-NN classification accuracy]{\label{fig:toy_accu}\includegraphics[width=0.32\textwidth]{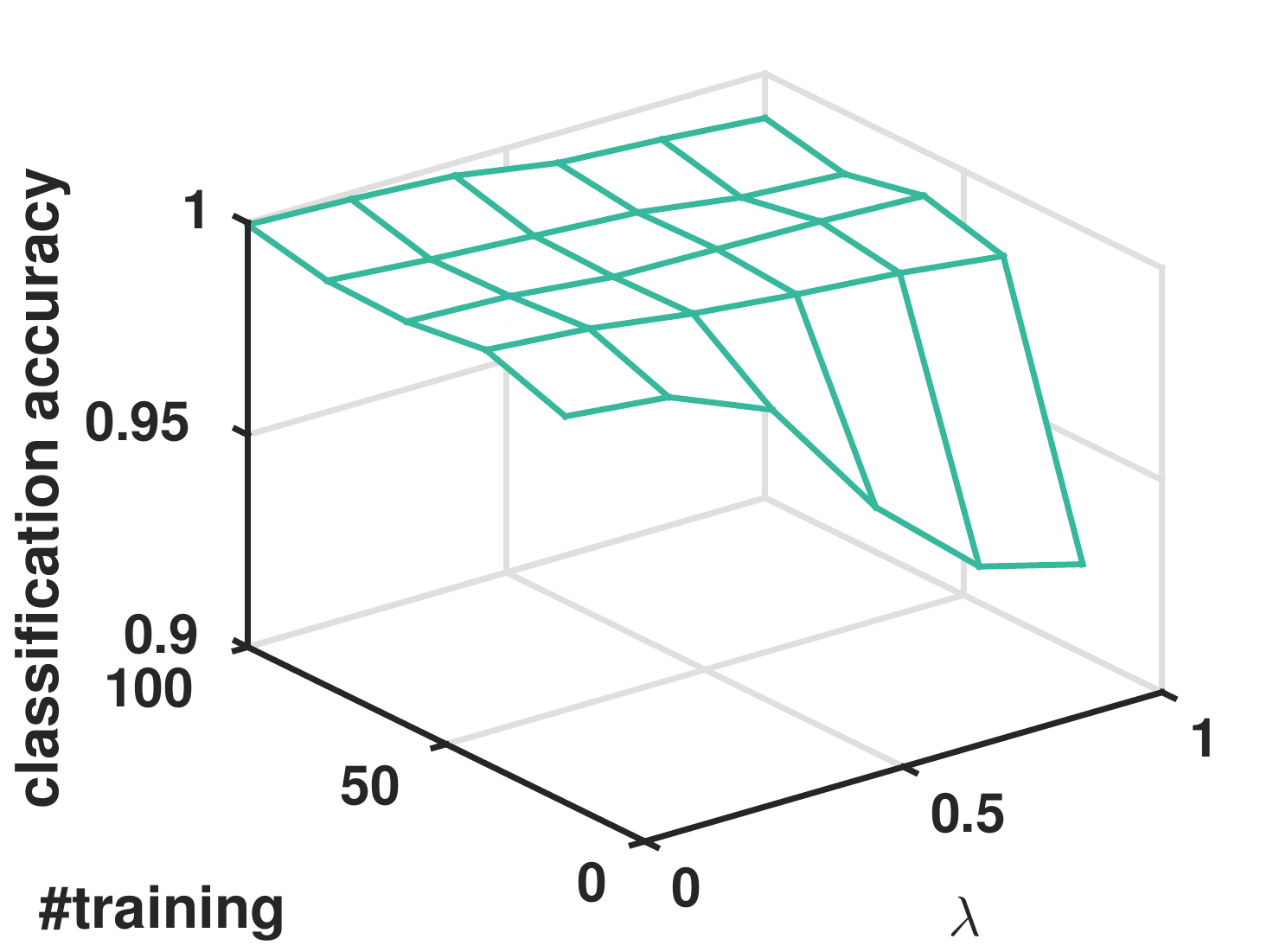}}
	\caption{Motivating example revisited. }
\end{figure*}

\subsection{Illustrative example revisited}
We provide here more experimental evaluation using the illustrative example in Section~\ref{sec:gdt}.
First we look at how $\lambda$ influences the performance.
The number of training samples per class ranges from 40 to 100. And $\lambda$ is varied in the $[0,1]$ interval.

Denote the training set as $\T$ and the testing set as $\V$.
In our case, the empirical loss on the training set is
\ben
R_{emp} = {1\over Z_\T}\sum_{\substack{i\neq j\\\x_i,\x_j\in\T}} \left(\frac{f_{\balpha_\T}(\x_i)^\top f_{\balpha_\T}(\x_j)}{\|f_{\balpha_\T}(\x_i)\|\cdot\|f_{\balpha_\T}(\x_j)\|}-\ell_{i,j}\right)^2,
\een
where $Z_\T$ is the  number of pairs constructed from the training set.
Note that the loss is not the objective in GDT formulation~\eqref{eq:obj} averaged over $Z_\T$;
 the objective of GDT incorporates an intra-class structure-preserving regularization, which should be excluded in evaluating the empirical loss.
The expected loss is empirically evaluated over the testing set, 
\ben
\hat R = {1\over Z_\V}\sum_{\substack{i\neq j\\\x_i,\x_j\in\V}} \left(\frac{f_{\balpha_\T}(\x_i)^\top f_{\balpha_\T}(\x_j)}{\|f_{\balpha_\T}(\x_i)\|\cdot\|f_{\balpha_\T}(\x_j)\|}-\ell_{i,j}\right)^2,
\een
where $Z_\V$ is the number of pairs constructed from the testing set.
Here we use the notation $\hat R$ to indicate that it is an empirical  estimate.

Fig.~\ref{fig:loss} shows  $R_{emp}$ and $\hat R$ for a variety of $\lambda$ and $|\T|$.
Note that the smaller $\lambda$ is, the more the structure-preserving regularization is emphasized.
We observe that  $R_{emp}$ is constantly lower than $\hat R$, indicating that $R_{emp}$ always tends to be optimistic.
As $|\T|$ increases, $\hat R$ decreases and $R_{emp}$ approaches $\hat R$.
Note that when $|\T|$ is small and $\lambda$ is big, $R_{emp}$ significantly underestimates $\hat R$.
Fig.~\ref{fig:genErr} shows an empirical estimate of the generalization error, $R_{emp}-\hat R$.
Fixing a particular $|\T|$, the generalization error decreases as $\lambda$ approaches zero, implying more robustness.

To see how the robustness influences classification, we apply a nearest neighbor (1-NN) classifier to the transformed testing data. The obtained classification accuracy is shown in Fig.~\ref{fig:toy_accu}. When the number of training samples per class is small, there is a steady increase in classification accuracy as $\lambda$ decreases, i.e., when more structure preservation is enforced;
and such increase becomes less obvious when the training set size increases.
The above observation clearly shows that, when only a small training set is given,  the robustness gained from the structure preservation dominates the classification performance.

As discussed before,  when $\lambda=0$, the objective function is optimized for classification by imposing
explicit constraints, $t_{i,j}=-1$ for negative pairs,  to separate different classes; however,
due to the structure preservation, weak constraints are used to enforce similar representation for
the same class.
This drawback cannot be overlooked for applications where it is critical to expect similar representations for the same class samples, such as face verification, and image retrieval.  In the next section, we use face verification to demonstrate a scenario where the balance between robustness and discrimination is preferred.

\subsection{MNIST}
The last section shows an extreme case where the best classification performance is achieved when $\lambda=0$. However, in general, $R$ takes minimum at a nontrivial $\lambda\in (0, 1)$, as illustrated in this section. We apply GDT to MNIST dataset. The $f_\balpha(\cdot)$ we adopted is a neural network made up of 3 convolutional layers. Between every two consecutive convolutional layer is a pooling layer. The original $28\times 28$ images are mapped to $256$ dimensional feature vectors. 

We vary $\lambda\in[0,1]$ and evaluate $R_{emp}$ on a small training set of size 500 (50 samples per class). $R$ is empirically estimated on testing set of size 10000. As shown in Fig.~\ref{fig:minist}, 
we observe that as $\lambda$ varies from 0 to 1, the empirical loss keeps decreasing (Fig.~\ref{minist_Remp}), implying increasing discrimination on training set. However, the generalization error keeps increasing (Fig.~\ref{minist_genErr}), implying decreasing robustness. Therefore, to achieve smallest $R$ (corresponding to best performance in testing set), we need to balance between discrimination and robustness. And in general, the $R$ takes minimum at some $\lambda\in(0,1)$ (Fig.~\ref{minist_R}).

\begin{figure*}[ht!]
	\subfloat[Empirical loss $R_{emp}$]{\label{minist_Remp}\includegraphics[width=0.24\textwidth]{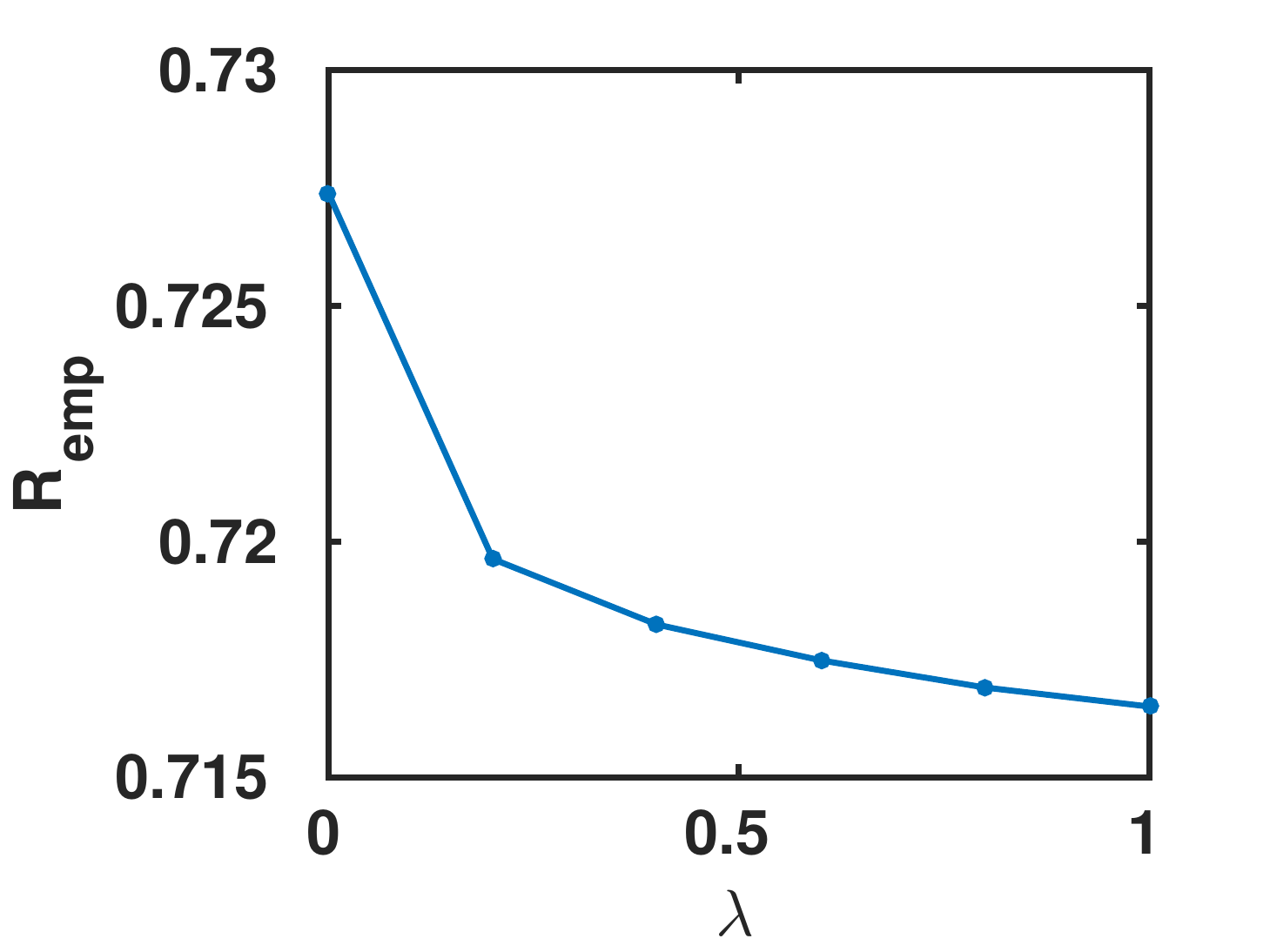}}~
	\subfloat[Generalization error]{\label{minist_genErr}\includegraphics[width=0.24\textwidth]{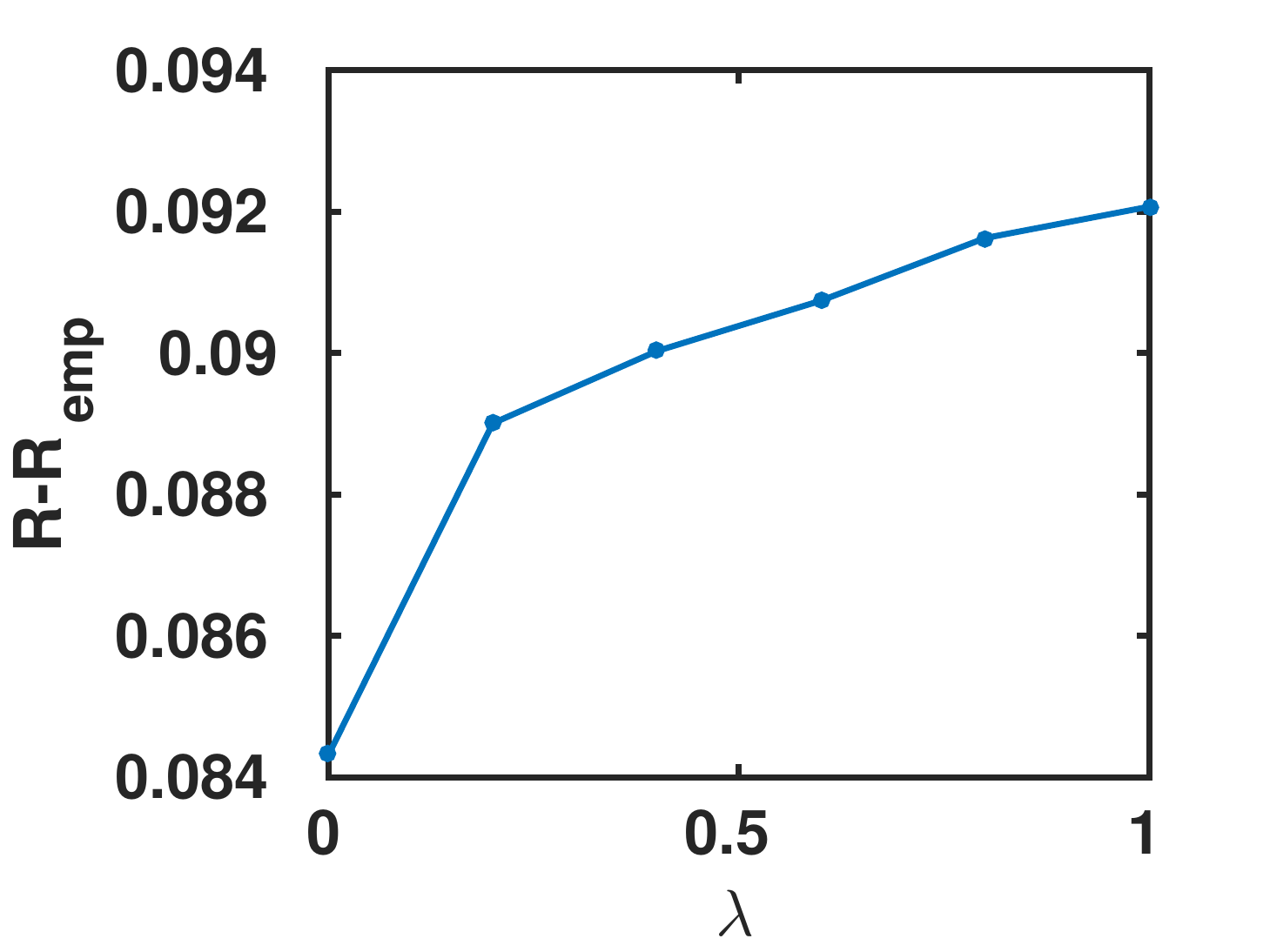}}~
	\subfloat[Test loss $R$]{\label{minist_R}\includegraphics[width=0.24\textwidth]{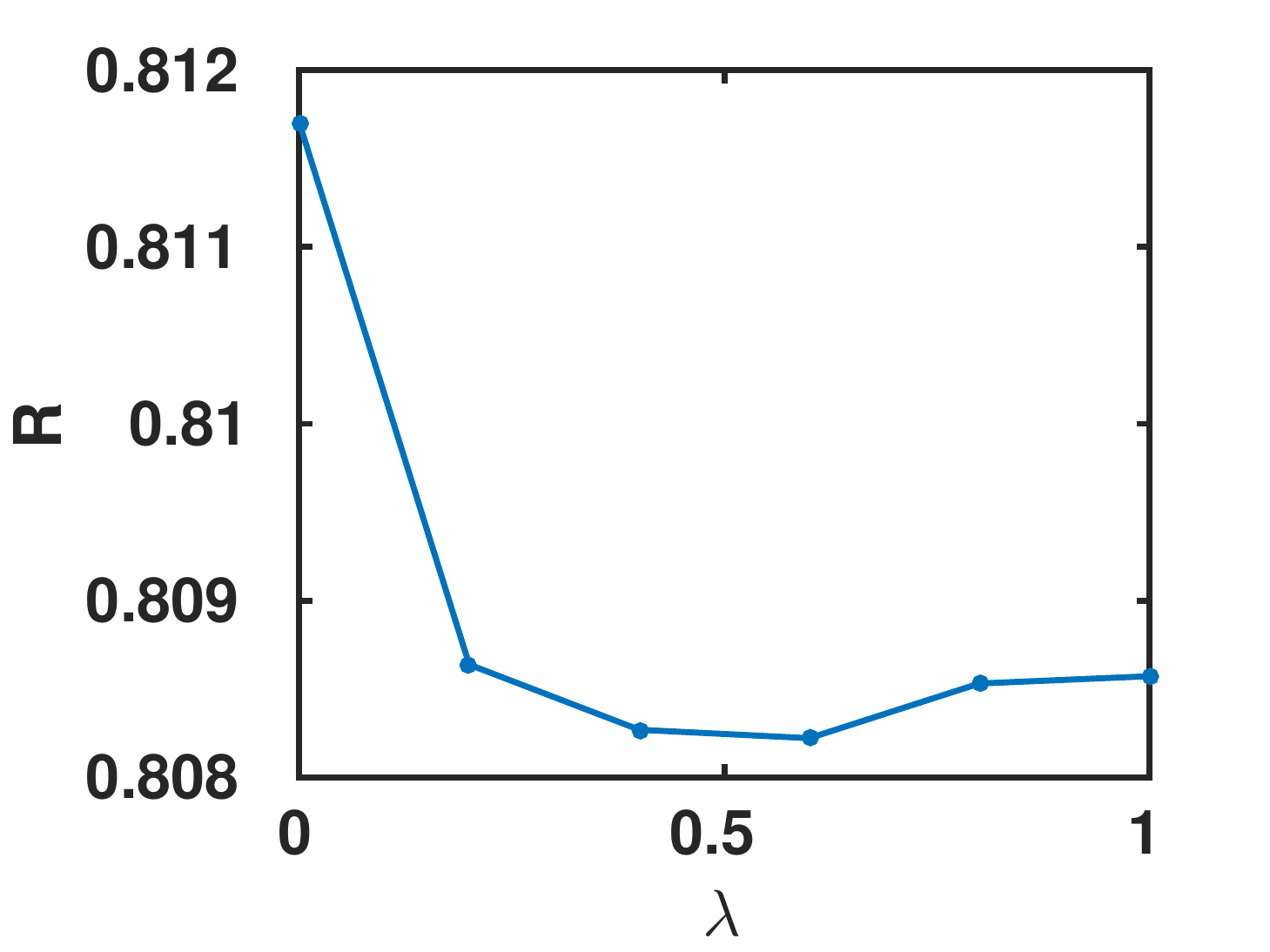}}
	\subfloat[Classification accuracy]{\label{mnist_accu}\includegraphics[width=0.24\textwidth]{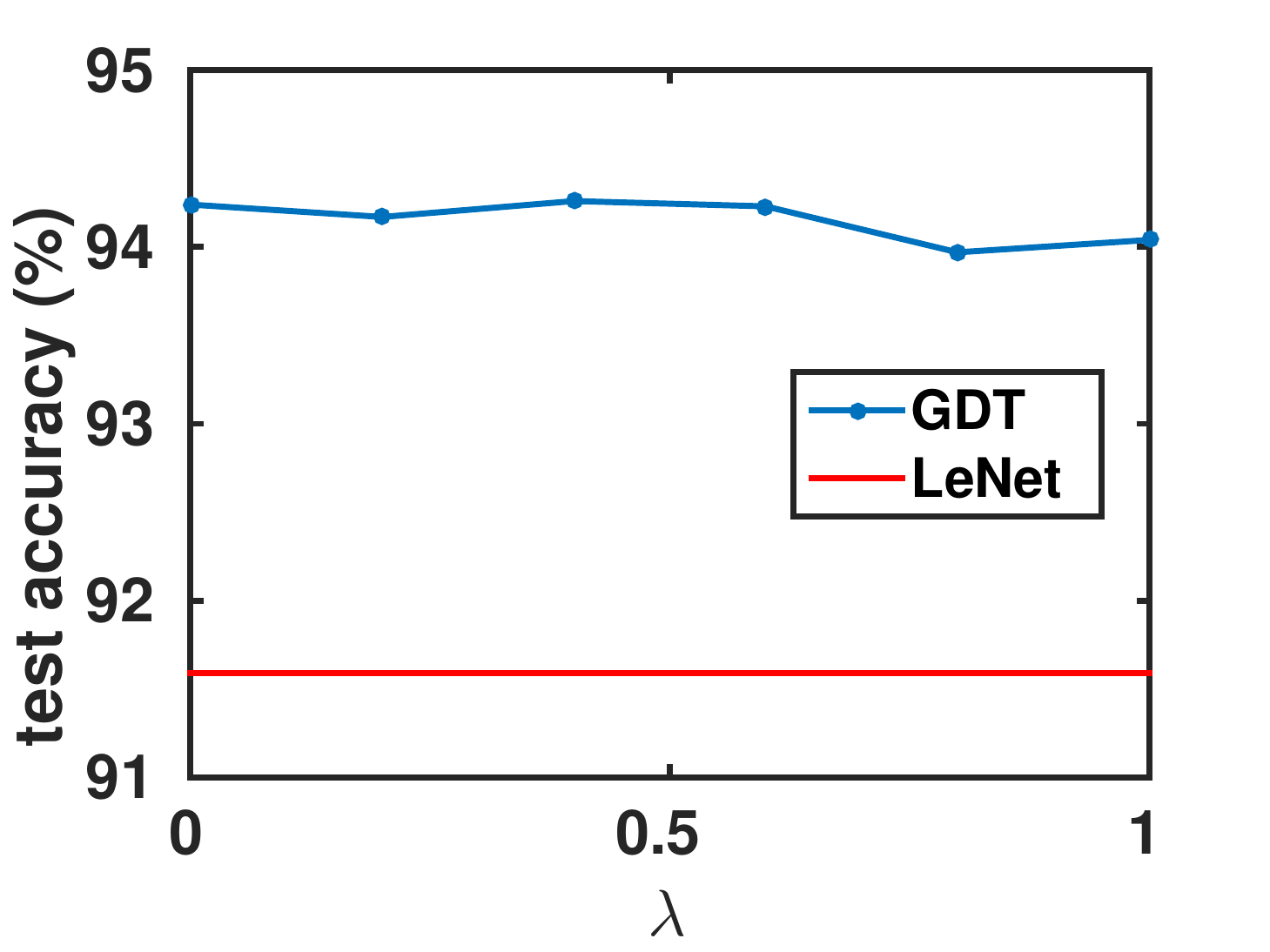}}
	\caption{GDT on MNIST with very small training set }
	\label{fig:minist}
\end{figure*}

As a comparison, we also run LeNet on the same training set. The LeNet's network structure is the same as the one adopted by GDT except that a fully connected layer and a softmax loss layer is added on the top. Fig.~\ref{mnist_accu} compares the classification accuracy of 1-nn on GDT features and that of LeNet. GDT's accuracy constantly outperforms LeNet and peaks around $\lambda=0.5$ where $R$ is the smallest.

\subsection{LFW}
\begin{figure*}[ht]
\center
	\subfloat[Comparison with state-of-art.]{\label{fig:lfw_compare}\includegraphics[width=0.42\textwidth]{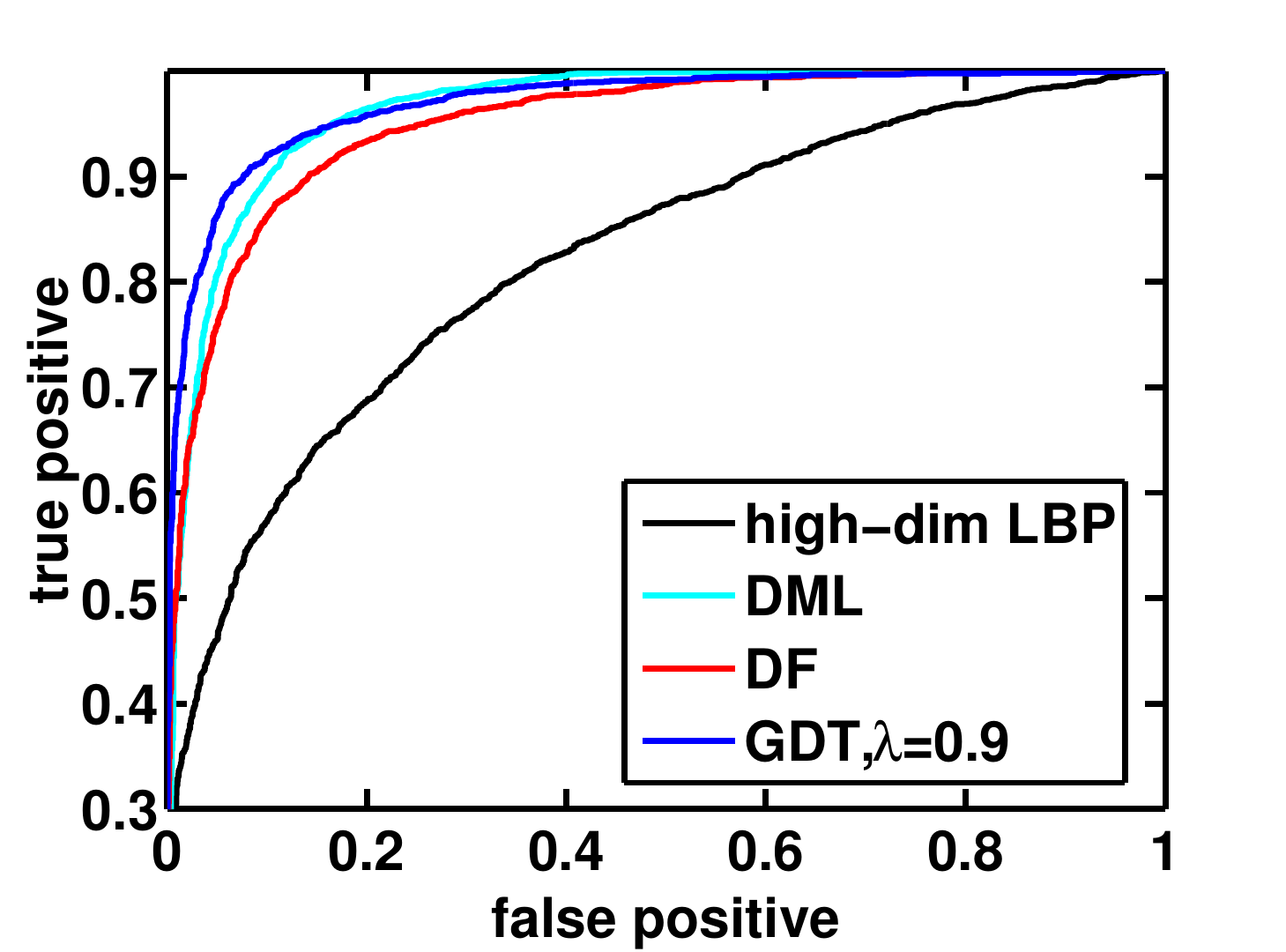}}
	\subfloat[Varying $\lambda$ in GDT.]{\label{fig:lfw_lambda}\includegraphics[width=0.42\textwidth]{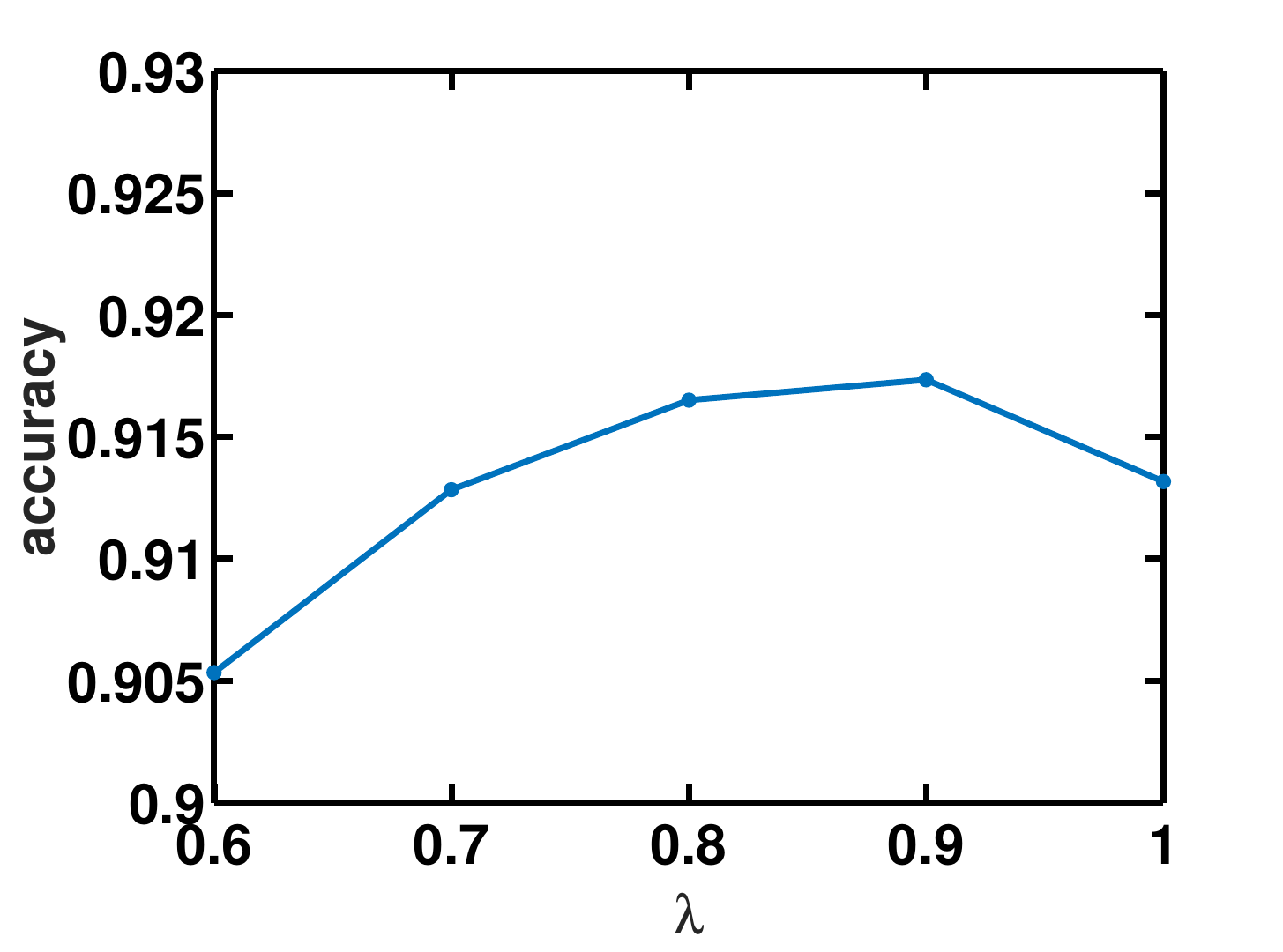}}
	\caption{Verification accuracy on LFW.}
\end{figure*}

We further validate the effectiveness of the geometry-aware deep transform by performing face verification on the challenging LFW benchmark dataset \cite{LFWref}.
Deep learning methods for face verification mostly use proprietary training data \cite{deepID2, deepID,deepface} and are therefore not reproducible.
We adopt the experimental setting from \cite{Bface2012}, and  train a deep network on the WDRef dataset~\cite{Bface2012}.
The WDRef dataset contains 2995 subjects and about 20 samples per subject, which is significantly smaller than a typical (proprietary) training set for deep learning, e.g., 4.4 million labeled faces from 4,030 people in \cite{deepface}, or 202,599 face images from 10, 177 subjects in \cite{deepID2}.
The goal of this paper is not to reproduce the success of deep learning in face verification \cite{DML, deepface}, but to
compare the proposed GDT with several popular objectives optimized in a deep network.
In our experiment, each face is described using a high dimensional LBP feature ~\cite{highDimChen2013} available at \cite{datalink}, which is reduced to dimension 5,000 using PCA.

We compare the proposed GDT with two state-of-the-art deep learning objectives: DeepFace (DF)~\cite{deepface}, and  Deep Metric Learning (DML)~\cite{DML}.
To enable a fair comparison, we adopt the same network structure and input features for all compared methods, but  keep their respective objective functions.
DF feeds the output of the last layer to a K-way soft-max to predict the probability distribution over K classes, and minimizes a softmax loss.
DML uses the Euclidean distance metric, and minimizes the loss defined in \eqref{eq:ml2}.
The function $f_\balpha(\cdot)$ in \eqref{eq:obj} is implemented as a two-layer fully connected network with $\tanh$ as the squash function, and the same network structure is used for DF and DML.
Weight decay (conventional Frobenius norm regularization) is adopted in both DF and DML. And a range of weight decaying factor is tried and the best testing performance is reported.
The network is trained on WDref and then applied to the LFW. To reflect the discriminability of the transformed features, we only use a simple verification method, by comparing the cosine distance between a given face pair to a threshold.

\begin{table}[h!]
	\centering
	\caption{Verification accuracy and AUC on LFW}
	\begin{tabular}{c|c|c}
		Method & accuracy (\%)& AUC \\
		\hline
		High-dim LBP &  74.73 & 0.8222$\pm 0.01$ \\
		\hline
		DF & 88.72 & 0.9550$\pm 0.0029$ \\
		\hline
		DML & 90.20 & 0.9640$\pm 0.0027$ \\
		\hline
		GDT & {\bf 91.72} & {\bf 0.9724$\pm 0.0029$}
	\end{tabular}
	\label{tab:lfw_stats}
\end{table}

The ROCs for all methods are reported in Fig.~\ref{fig:lfw_compare}. Verification accuracies and area under the ROC curves (AUC) are listed in Table~\ref{tab:lfw_stats}.
High-dim LBP denotes the original features before transform.
DF optimizes for a classification objective,  the softmax loss, and separates well samples from different classes; however, it enforces no explicit constraints to assign similar representations to the same class.
DML enforces discriminative pairwise distance; but, as illustrated before, becomes less robust when restricted to a small training set.
As analyzed in Section \ref{sec:theory}, the proposed GDT is less conservative than DF for better discriminability; and, at the same time,
expects smaller generalization errors than DML by preserving the local geometry \eqref{eq:obj}. We observe that GDT outperforms both DF and DML by achieving a balance between discrimination and robustness.
Face verification accuracies are shown in Fig.~\ref{fig:lfw_lambda} by varying $\lambda$ from $0.6$ to $1$; and  peak accuracy is observed at $\lambda=0.9$, illustrating  the effectiveness of geometry preservation.
Considering the facts that both DF and DML are state-of-art deep learning methods, the improvements reported here clearly demonstrate the strength of GDT.

We demonstrated here how the discriminability of original features, e.g., high-dim LBP here, can be improved with a learned feature transform.
As emphasized,  the goal is not to reproduce the success of deep learning in face verification (which can't be done due to the lack of availability of the data used in the corresponding papers); thus, we perform verification by simply comparing the cosine distance between each pair with a threshold. Note that more advanced verification techniques such as \emph{JointBayes} \cite{Bface2012}  can always be adopted for improved accuracies; for example, \cite{highDimChen2013} reports 95.17\% accuracy by applying the JointBayes method on the high-dim LBP features. As observed in \cite{deepID2}, we also expect steady improvements in verification accuracy by increasing the number of subjects used in training a deep network.

\section{Conclusion}
We proposed a geometry-aware deep transform that unifies both the classification and metric learning objectives commonly optimized in learning a deep network.
We provided both experimental and theoretic illustrations to show that our method achieves a balance between  discrimination and robustness,
especially when restricted to a small training set.  We demonstrated the effectiveness of the proposed deep learning objective using real-world data for applications such as face verification.

\section*{Acknowledgement}
Work partially supported by NSF and DoD. 

{
\bibliographystyle{ieee}
\bibliography{geoAware}
}

\end{document}